\documentclass[11pt]{article}

\usepackage{bm}
\usepackage{amsfonts,amsthm,amsmath,amssymb,mathtools}
\usepackage{fullpage}
\usepackage{thmtools}

\usepackage[sort&compress]{natbib}
\usepackage{algorithm,algorithmicx}
\usepackage{subcaption}
\usepackage{graphicx}
\usepackage{color}
\usepackage[export]{adjustbox}
\usepackage{authblk}

\usepackage{xcolor}
\colorlet{blue}{black}

\allowdisplaybreaks[4]

\usepackage{url}
\usepackage{algcompatible}
\usepackage{algpseudocode}

\def\w{{\bf w}}

\def\x{{\bf x}}
\def\X{{\bf X}}

\def\v{{\bf v}}
\def\s{{\bf s}}
\def\1{{\bf{1}}}
\def\0{{\bf{0}}}

\def\X{{\bf X}}
\def\D{\mathcal{D}}

\def\AUC{{\sc \bf AUC}}
\def\BAUC{{\sc \bf BAUC}}
\def\H{\mathcal{H}}
\def\ERR{{\sc \bf ERR}}

\def\bepsilon{\bm{\epsilon}}

\def\1{{\bf{1}}}
\def\0{{\bf{0}}}

\def\w{{\bf w}}

\def\v{{\bf v}}

\def\x{{\bf x}}

\def\g{{\bf g}}

\def\bc{\textcolor{blue}}

\newcommand{\R}{\mathbb{R}}

\newcommand{\E}{\mathbb{E}}

\DeclareMathOperator*{\argmin}{arg\,min}

\usepackage[figuresright]{rotating}
\usepackage[ruled,vlined,algo2e]{algorithm2e}

\newcounter{thm_counter}

\newtheorem{theorem}[thm_counter]{Theorem}

\begin{document}

\title{A Robust AUC Maximization Framework with Simultaneous Outlier Detection and Feature Selection for Positive-Unlabeled Classification}

\author[1]{Ke~Ren$^*$}
\author[1]{Haichuan~Yang\thanks{Equal contribution.}}
\author[1]{Yu~Zhao}
\author[2]{Mingshan Xue}
\author[3]{Hongyu Miao}
\author[4]{Shuai~Huang}
\author[1]{Ji~Liu}

\affil[1]{\small Department of Computer Science and Department of Electrical and Computer Engineering, University of Rochester, Rochester, NY 14627, USA}
\affil[2]{\small Baylor College of Medicine, Houston, TX 77030, USA}
\affil[3]{\small University of Texas Health Science Center at Houston, Houston, TX 77030, USA}
\affil[4]{\small Department of Industry and Systems Engineering, University of Washington, Seattle, WA 98195, USA}

\date{\today}

\maketitle

\begin{abstract}
The positive-unlabeled (PU) classification is a common scenario in real-world applications such as healthcare, text classification, and bioinformatics, in which we only observe a few samples labeled as ``positive'' together with a large volume of ``unlabeled'' samples that may contain both positive and negative samples. Building robust classifier for the PU problem is very challenging, especially for complex data where the negative samples overwhelm and mislabeled samples or corrupted features exist. To address these three issues, we propose a robust learning framework that unifies AUC maximization (a robust metric for biased labels), outlier detection (for excluding wrong labels), and feature selection (for excluding corrupted features). The generalization error bounds are provided for the proposed model that give valuable insight into the theoretical performance of the method and lead to useful practical guidance, e.g., to train a model, we find that the included unlabeled samples are sufficient as long as the sample size is comparable to the number of positive samples in the training process. Empirical comparisons and two real-world applications on surgical site infection (SSI) and EEG seizure detection are also conducted to show the effectiveness of the proposed model.
\end{abstract}

\section{Introduction}
{T}{he} positive-unlabeled (PU) classification is quite common in many real-world applications such as healthcare \citep{kaur2006empirical, peng2009healthcare}, text classification \citep{icml2002partially, kdd2002PEBL}, time series classification \citep{ijcai2011positive}, and bioinformatics \citep{bioinf2012positive}. The PU classification is defined as: given a few samples labeled as ``positive'' and a high volume of ``unlabeled'' samples that contain both negative and positive samples, a binary classifier is learned from it. Existing works for PU classification \citep{geo2007a, geo2011a, icml2003leaning, sdm2009positive, bioinf2012positive, elkan2008learning} work well for well-conditioned data. However, some important issues remain unsolved for information-rich but complex data such that the performance is often unsatisfactory. In particular, we consider the following three issues for complex data:
\begin{itemize} 
\item The data is highly biased -- negative samples dominate. \bc{For example, in the abnormal event detection, e.g., earthquake detection, seizure detection (considered in our experiment), there are no more than 0.1\% of samples are positive. In this case, a dummy classifier classifying all data to be negative achieves prediction accuracy $99.9\%$. As a result, the commonly used prediction error is no longer a robust and stable objective or metric.}
\item The labeled positive samples includes contain incorrect labels. \bc{It happens often in practice when the labeled is provided by humans. Although the percentage of samples with wrong labels (we called outliers in this paper) could be small, such outliers could seriously influence the performance.}
\item \bc{Samples are with redundant or irrelevant features. Redundant features could seriously hurt the performance due to causing overfitting issue, especially when the number of samples are very limited.}
\end{itemize}

\bc{Although each single issue has been considered in existing literature \cite{du2014analysis,hodge2004survey,zhang2009adaptive}, a naive step-by-step solution compiling existing methods does not work well, since any one issue can affect another two. For example, if we first apply a feature selection algorithm, followed by an outlier detection algorithm, the selected features in the first step could be totally wrong because it does not know there may exist outliers; vise versa. Therefore, only when all issues are considered simultaneously, it is possible to find out correct features and outliers. Unfortunately, little is known about how these three issues can be handled simultaneously, which makes the PU problem on complex data more challenging. It motivates us to find out an integral solution to solve all three issues together.}

To address these three issues jointly on PU classification, we propose a robust learning framework that unifies AUC (area under the curve) maximization -- a robust metric for biased labels, outlier detection (for excluding wrong labels and bad samples), and feature selection (for excluding corrupted features). Firstly, existing works for PU classification \citep{geo2007a, geo2011a, icml2003leaning, sdm2009positive, bioinf2012positive, elkan2008learning} mainly use the misclassification error or the recall value as the performance metric to guide the training of the classifier. However, \bc{under the settings of the PU learning, highly biased training set, AUC serves as a more robust metric  \citep{cortes2004auc} since it is invariant to the percentage of positive samples. Although AUC optimization has been studied before, i.e. \citep{rakotomamonjy2004optimizing,brefeld2005auc}, it cannot be directly adopted the AUC metric to the PU problem. This is because negative samples are unavailable in the PU's scenario. To overcome this difficulty, }  we propose to use a blind AUC (BAUC) criterion to approximate the target AUC, and we will show that in theory maximizing BAUC is equivalent to maximizing AUC. Secondly, as the ``unlabeled'' samples are actually a heterogeneous collection of samples that contain outliers, the learning formulation needs to automatically identify the outliers and leave them out of the training process of the classifier. As a matter of fact, as the name ``unlabeled'' suggested, positive samples with wrong labels or positive samples with corrupted feature values, are likely to happen and contribute to the ``outlier'' category of the unlabeled samples. Thirdly, as feature selection has been a critical aspect for mitigating the overfitting issue, we also need to ensure that our learning formulation is able to incorporate this functionality. 

The main contributions in this paper are summarized in the following:
\begin{itemize} 
    \item The proposed $\BAUC$-OF model is more robust than existing classification error minimization frameworks (for example, \cite{plessis2015convex}), particularly in these two aspects: 1) there is no need to set a prior value for the percentage of positive samples $\pi$ in the training process, which has been a difficulty in many applications; and 2) both outlier detection and feature selection are integrated with the AUC maximization formulation.
    \item The generalization error bounds are also provided for the proposed model. It gives valuable insights into the theoretical performance of the method and reveals relationships between some important parameters (such as the dimensionality of the features, the samples sizes for both positive samples and ``unlabeled'' samples) of the PU problem with the performance of the learned classifier. Those insights also lead to useful guidance for practices such as that the unlabeled samples are sufficient as long as the sample size is comparable to the number of positive samples in the training process.
    \item Empirical experimental studies have been conducted on a thorough collection of datasets that demonstrate the proposed method outperforms the state-of-the-art approaches. 
    \item Last but not least, it is worthy of mentioning that the proposed outlier detection and feature selection technique can be easily extended to other formulations. That means, although our proposed $\BAUC$-OF model is motivated for general PU problems, the proposed outlier detection and feature selection can also be integrated with other PU learning formulations that have been developed for specific PU problems. To the best of our knowledge, this is the first work to simultaneously select features and identify outliers.
\end{itemize}

\section{Related Works} 
\label{rel-work}
This section reviews related works about PU problem, outlier detection, and feature selection.

{\noindent \bf Learning From PU Data} 
PU learning is mainly about learning a binary classifier from a dataset containing both positive and unlabeled data \citep{elkan2008learning}. \bc{The labeled positive data is assumed to be selected randomly from the population. The traditional approach to solving the PU learning is to simply treat all unlabeled data as negative samples, which may result in biased solutions. To mitigate this bias, several methods are proposed.  One-class classification \citep{moya1996network} only uses the positive samples in the training set. These works include \cite{de2007kernel,manevitz2001one}. }  \bc{ Further some non-convex loss functions are introduced to mitigate this bias, for example \citep{smola2009relative,du2014analysis}. Finally \citep{plessis2015convex} propose the convex formulation that can still cancel this bias for PU problem.}  In some applications where the class prior is known, the PU learning can be reduced to solving a cost-sensitive classification problem \citep{elkan2001foundations}. Works focused on adjusting the weights inside the loss functions according to the class prior are also studied in \cite{scott2009novelty,blanchard2010semi,li2003learning,icml2003leaning}. 
However, an inaccurate estimation of the class prior will increase the classification error. Thus some research efforts in the PU problem also investigated different ways to estimate the class prior such as \citep{elkan2008learning, blanchard2010semi, du2014class}. Other approaches proposed  also include graph-based approach \cite{pelckmans2009transductively}, bootstrap-based approach \cite{mordelet2014bagging}. We omit the details here. \bc{In this paper, we propose an AUC-based PU learning framework where the AUC metric is used to guide the learning process. We show this robust metric is especially suitable for PU learning and can be integrated with outlier detection and feature selection to achieve better performance than state-of-art PU learning approaches.}

{\noindent \bf Outlier Detection} 
To the best of our knowledge, the outlier issue has not been considered in PU classification. Thus, here we provide a brief overview of outlier detection methods in general settings. \bc{Existing outlier detection researches can be roughly categorized as variance-based approaches and model-based approaches. The variance-based approaches detect the outliers through a set of criterion used to measure the difference between the data and the rest of the dataset. These criterion can come from statistical analysis \citep{yamanishi2001discovering,yamanishi2004line}, distance metric \citep{knorr1999finding} and density ratio \citep{aggarwal2001outlier,jiang2001two,breunig2000lof}. They usually work well in low dimensional space when the data amount is not too huge. But they are difficult to be integrated into other models like PU learning. The model-based approaches detect outliers through certain models. Typical methods include the regularized principal component regression \citep{walczak1995outlier}, regularized partial least square \citep{hubert2003robust}, SVM \citep{jordaan2004robust}  based algorithms and others. More detailed reviews can be found in \cite{hodge2004survey}.  None of these works were integrated with PU learning models.}

{\noindent \bf Feature Selection / Overfitting} 
Feature selection is very important in machine learning especially for applications where the number of features is much larger than the number of data points. While feature selection has a wide spectrum of approaches, \bc{for example \citep{he2006laplacian,yao2017lle}}, here, the sparse learning area is more related to our study. Generally, there are two types of sparse learning approaches. One type of approaches include the convex relaxation formulations, represented by the  $\ell_1$ norm based approaches such as the LASSO formulation developed in \citep{tibshirani1996regression} and further extended in \citep{ng2004feature, zou2006the, friedman2008sparse}. The other approach is the non-convex formulation, represented by the $\ell_0$ norm based (or greedy) approaches such as OMP \citep{tropp2004greed}, FoBa \citep{zhang2009adaptive, liu2013forward}, and projected gradient based methods \citep{yuan2013gradient, nguyen2014linear}. Usually, the $\ell_1$ based methods are convex relaxations of the $\ell_0$ norm based methods. Theoretical studies have also been conducted to compare the performance of the $\ell_0$ norm based approaches with $\ell_1$ based methods such as \citep{zhang2009adaptive, liu2013forward}. \bc{ The feature selection framework is also extended to the multi-class settings  \citep{obozinski2006multi,chapelle2008multi,xu2017feature}.}

\section{Blind AUC Formulation with Outlier Detection and Feature Selection} \label{Sec:model}

We first propose the blind $\AUC$ ($\BAUC$) metric for the PU problem and show its connection with the AUC for binary classification problem in Section~\ref{Sec:model:BAUC}. Then, Section~\ref{sec:model:BAUCOF} introduces the proposed formulation that unifies the BAUC maximization with simultaneous outlier detection and feature selection. \bc{The outlier detection and the feature selection are further integrated with other PU formulations.} In Section~\ref{sec:model:opt}, we develop the optimization algorithm for solving the proposed model.

\subsection{Blind AUC (BAUC)}
\label{Sec:model:BAUC}

 The $\AUC$ \citep{rad1982the, qua2002areas} metric is defined as
\begin{align*}
\AUC(f) = \E_{\x_+\sim \D_+} \E_{\x_- \sim \D_-} \1 (f(\x_+) \geq f(\x_-)),
\end{align*} 
where $f$ is a scoring function, for example, $f(\x) = \w^\top \x$. $\D_+$ and $\D_-$ denote the distributions for positive samples and negative samples respectively. Indicator function $\1(\cdot)$ returns value $1$ if the condition is satisfied; $0$ otherwise. Intuitively, $\AUC(f)$ measures the probability that the scoring of $\x_+$ is greater than $\x_-$ if $\x_+$ and $\x_-$ are randomly sampled from the positive class and the negative class. It has been known that $\AUC$ is a more stable and robust metric than accuracy for biased binary classification problem. Thus, it has been widely used to guide the training of the classification model $f$ in binary classification problem.

Although the maximization of AUC is our ultimate goal in the PU problem, it is hard to directly apply it because the negative labels are not available in the PU problem. Therefore, we consider a blind AUC (BAUC) for the PU problem. In particular, BAUC simply \emph{blindly} treats all unlabeled samples as negative samples and defines $\BAUC$ in the following
\begin{align*}
\BAUC(f) = \E_{\x_+\sim \D_+} \E_{\x \sim \D} \1 (f(\x_+) \geq f(\x)).
\end{align*}
where $\D$ is the distribution for unlabeled samples.
Thus, using the $\BAUC$, we can derive the following empirical formulation to learn the classifier $f$ from the positive training set $\X_+$ and unlabeled training set $\X$:
\begin{align}
\widehat{\BAUC}(f):= {1\over |\X| |\X_+|}\sum_{\x\in \X, \x_+\in \X_+} \1 (f(\x_+) \geq f(\x)). 
\label{eq:BAUC:e}
\end{align}
It is easy to verify that
$$
\E[ \widehat{\BAUC}(f)] = {\BAUC}(f) 
$$ \bc{from the fact that  the expected value of the sum of random variables is equal to the sum of their individual expected values. In the proof of Theorem~\ref{thm:VC}, we will show the details.}

Note that, one can approximately maximize \eqref{eq:BAUC:e} by replacing the indicator function $\1(\cdot)$ by a surrogate function, e.g, hinge loss or logistic loss. 

Although we can only empirically maximize $\BAUC$, the following Theorem ~\ref{thm:AUCandBAUC} actually suggests that maximizing $\BAUC$ essentially maximizes $\AUC$ (recall that it is our ultimate goal). Particularly, Theorem~\ref{thm:AUCandBAUC} reveals that $\AUC$ depends on $\BAUC$ linearly, which indicates that when $\BAUC$ achieves the maximum, $\AUC$ achieves its maximum too.

\begin{theorem} \label{thm:AUCandBAUC}
For binary classification problem, given an arbitrary scoring function $f$, there exists a linear dependence between its $\AUC(f)$ value and its $\BAUC(f)$ value:
\begin{equation*}
\AUC(f) = (1-\pi)^{-1}(\BAUC(f)-{\pi /2}),
\end{equation*}
where $\pi$ is the percentage of positive samples.
\end{theorem}
\begin{proof} 
From the definition of $\BAUC$, we have
\begin{align*}
\BAUC(f) &= \E_{\x_+\sim \D_+} \E_{\x \sim \D} \1 (f(\x_+) \geq f(\x))\\
&= \E_{\x_+\sim \D_+} \E_{{\x_- \sim \D_-},{\x'_+ \sim \D_+}} \1 (f(\x'_+) \geq f(\x_-))\\
&= (1-\pi)\E_{\x_+\sim \D_+} \E_{{\x_- \sim \D_-}} \1 (f(\x_+) \geq f(\x_-)) + \pi\E_{\x'_+\sim \D} \E_{\x_+ \sim \D} \1 (f(\x'_+) \geq f(\x_+)).
\label{auc-exp}
\end{align*} 
The term $\E_{\x'_+\sim \D} \E_{\x_+ \sim \D} \1 (f(\x'_+) \geq f(\x_+))$ is a constant, because the probability that a randomly chosen positive sample is ranked higher than another randomly chosen positive sample from one same data set should always be $1/2$. So we have : 
\begin{equation*}
\BAUC(f)= (1 - \pi)\AUC(f) + \pi /2.
\end{equation*}
which proves the theorem.
\end{proof}

\subsection{Integration of $\BAUC$ Maximization with Outlier Detection and Feature Selection} \label{sec:model:BAUCOF}

The maximization of $\widehat{\BAUC}(f)$ is equivalent to minimizing $1-\widehat{\BAUC}(f)$:
$$
\min_{f} 1-\widehat{\BAUC}(f) = {1 \over |\X||\X_+|} \sum_{\x\in \X, \x_+\in \X_+} \1 (f(\x_+) < f(\x)).
$$

While the concept of outlier has been diverse, in this paper, we mainly consider the outliers that include: 1) the samples that are wrongly labeled as positive; 2) the samples in the positive samples whose feature values are corrupted for whatever reasons, whose existence will distort the distribution of the data points. To identify those outliers, we construct a vector $\bepsilon \in \R^{|\X_+|}$ while each positive sample $\x_+$ corresponds to a coordinate of $\bepsilon$, denoted by $\bepsilon_{\x_+}$. With this notation, the following optimization formulation can be derived:
\begin{align}
\min_{f, \bepsilon}\quad & {1 \over |\X||\X_+|} \sum_{\x\in \X, \x_+\in \X_+} \1 (f(\x_+) + \bepsilon_{\x_+}< f(\x))
\\
\text{s.t.}\quad & \|\bepsilon\|_0 \leq t.
\label{eq:BAUC:1}
\end{align}
The key motivation behind this formulation is to use $\bepsilon_{\x_+}$ to adjust the score of the outlier $\x_{+}$ instead of modifying its label feature values though it actually has the equivalent effect. The constraint in \eqref{eq:BAUC:1} is to restrict the maximal number of outliers by a user-defined parameter $t$, and the nonzero elements of the optimal $\bepsilon$ indicate outliers. To the best of our knowledge, this is the first time to apply the $\ell_0$ norm for outlier detection while it has been used for the feature selection purpose before.

Next we integrate the $\BAUC$ model with feature selection capacity. The basic task is to restrict the hypothesis space for $f$. Here, we restrict our interest on linear scoring forms for $f$, that is, $f(\x) = \w^\top \x$ where $\w$ parameterizes the scoring function $f$. It is worthy of mentioning that our proposed framework could be extended to nonlinear models as well. Here, we mainly consider three types of sparsity hypothesis ($\H$) for $\w$ by defining $\H$ as
\begin{subequations}
\begin{align}
\label{eq:Hs:def}
\H_s = & \{\w\in \R^p~|~ \|\w\|_0 \leq s\}, \\
\label{eq:HGs:def}
\H_{\mathcal{G}, s} = & \left\{\w\in \R^p~\Big|~ \sum_{G\in \mathcal{G}}\1(\w_{G}\neq 0) \leq s\right\}, \\
\label{eq:HGS:def}
\H_{\mathcal{G}, \s} = & \{\w\in \R^p~|~\|\w_{G}\|\leq \s_G~\forall~G\in \mathcal{G}\},
\end{align}
\label{eq:H}
\end{subequations}
where $\mathcal{G}$ is the set of disjoint group index sets, $s$ is a scalar that specifies the upper bound of the feature size, and $\s \in \R^{|\mathcal{G}|}$ is a vector. $\H_s$ is the commonly used sparsity hypothesis space in sparse learning \citep{tibshirani1996regression}; $\H_{\mathcal{G}, s}$ denotes the group sparsity set \citep{huang2010the, zhang2010automatic}; and $\H_{\mathcal{G},\s}$ is the exclusive sparsity set enforcing the selection diversity \citep{campbell2015within}.

To put everything together, the final model can be summarized in the following:

\begin{align*}
\min_{\w, \bepsilon}\quad & {1 \over |\X||\X_+|} \sum_{\x\in \X, \x_+\in \X_+} \1 (\w^\top \x_+ + \bepsilon_{\x_+}< \w^\top \x)
\\
\text{s.t.}\quad & \|\bepsilon\|_0 \leq t, \quad \w \in \H.
\end{align*}

Since the indicator function $\1(\cdot)$ (or equivalently the $0-1$ loss function) is not continuous, the common treatment is to use convex and continuous surrogate function to approximate it, such as the hinge loss and logistic loss function. Without loss of generality, here, we focus on the logistic loss (similar algorithms and theories can be applied to other smooth loss functions). This leads to the following $\BAUC$ formulation with Outlier detection and Feature Selection (named as $\BAUC$-OF):
\begin{align}
\min_{\w,\bm{\epsilon}} \quad & 
\frac{\alpha}{2}\| \bm{\epsilon} \|^2+\frac{\beta}{2}\| \w \|^2 + {1\over |\X_+||\X|}\sum_{\x_+\in \X_+, \x \in \X} h(\x_+, \x; \w, \bepsilon_{\x_+}) 
\\
\text{s.t.}\quad & \| {\bm{\epsilon}} \|_0 \leq t, \w \in \H, 
\label{auc-smo}
\end{align}

where $h(\x_+, \x; \w, \bepsilon_{\x_+})$ is defined as 
$\log(1+\exp(-\w^{\top}(\x_+-\x)-\bm{\epsilon}_{\x_+})).$
Note that the two additional terms $\frac{\alpha}{2}\|\bm{\epsilon}\|^2$ and $\frac{\beta}{2}\|\w\|^2$ serve as the regularization term to cope with the possibility that $\bepsilon$ or $\w$ diverges in the optimization process. $\alpha$ and $\beta$ are usually set to be small values.

The proposed outlier detection and feature selection scheme is quite flexible and can easily incorporate with other existing PU frameworks.
Most existing models for PU problem consider minimizing the misclassification error. It is of interest to compare our proposed $\BAUC$-based models with these $\AUC$-based methods. Particularly, as we also integrate the $\BAUC$ model with outlier detection and feature selection, in this section, we further illustrate how the counterpart of $\BAUC$-OF can be developed in the framework of $\AUC$-based framework using a recent development  \citep{plessis2015convex}. Particularly, following \citep{plessis2015convex}, the error minimization is given as:
\begin{equation}
\label{error-obj}
\begin{aligned}
\ERR(\w) = & \pi\mathbb{E}_{\x_+ \in \X_+}\1(\w^\top \x_+ + b \leq 0) + (1-\pi)\mathbb{E}_{\x_- \in \X_-} \1(-\w^\top \x_--b \leq 0).
\end{aligned}
\end{equation}
Applying the logistic loss function to approximate the indication function $\1(\cdot \leq 0)$, \eqref{error-obj} can be written as \citep{plessis2015convex}:
\begin{equation*}
\label{error-out}
\begin{array}{l}
\begin{aligned}
\min_{\w,b}:\quad f(\w, b) := &\frac{\pi}{|X_+|}\sum_{\x_+\in \X^+}(\w^\top \x_++b)\\ 
&+\frac{1}{|X_+|}\sum_{\x_+\in \X^+}\log[1+\exp(b-\w^\top \x_+)]\\
&+\frac{1}{|X|}\sum_{\x \in \X}\log[1+\exp(-b-\w^\top \x)].
 \end{aligned}
\end{array} 
\end{equation*}
We then introduce the outlier detection and feature selection to this model to obtain the error minimization formulation with outlier detection and feature selection (named as $\ERR$-OF) in below:
\begin{align}
\nonumber
\min_{\w,b,\bm{\epsilon}}:\quad &\frac{\beta}{2}\|\w\|^2+\frac{\alpha}{2}\|\bm{\epsilon}\|^2 + 
\frac{\pi}{|X_+|}\sum_{\x_+\in \X^+}(\w^\top \x_++b-{\bepsilon_{\x_+}})\\ 
\nonumber
&+\frac{1}{|X_+|}\sum_{\x_+\in \X^+}\log[1+\exp(\bm{\epsilon}_{\x_+}+b-\w^\top \x_+)]\\
\nonumber
&+\frac{1}{|X|}\sum_{\x \in \X}\log[1+\exp(b-\w^\top \x)]\\
\text{s.t.} \quad & \|\bm{\epsilon}\|_0 \leq t, \quad \w \in \H.
\label{error-out}
 \end{align}

\subsection{Optimization} \label{sec:model:opt}
This section introduces the optimization algorithm to solve the proposed $\BAUC$-OF formulation in \eqref{auc-smo}. Eq.~\eqref{auc-smo} is a constrained smooth nonconvex optimization. The nonconvexity is due to the constraints for $\w$ and $\bepsilon$. A natural idea is to apply the commonly used projected gradient descent algorithm to solve it. However, the AUC formulation involves a huge number of interactive terms between positive samples and the unlabeled samples. To reduce the complexity of each iteration, we use the stochastic gradient to approximate the exact gradient. In particular, we iteratively sample $\x$ from $\X$ to calculate the unbiased stochastic gradient:
\begin{align*}
\g_{\w}:= &{1\over |X_+|} \sum_{\x_+ \in \X_+} \nabla_{\w} h(\x_+, \x; \w, \bepsilon_{\x_+}) + \beta \w,
\\
\g_{\bepsilon}:= &{1\over |X_+|} \sum_{\x_+ \in \X_+} \nabla_{\bepsilon} h(\x_+, \x; \w, \bepsilon_{\x_+}) + \alpha \bepsilon.
\end{align*}
and apply the projected gradient step to update the next iteration
\begin{align*}
\w^{(k+1)} = & \argmin_{\w\in \H} \|\w - (\w^{(k)} - \eta\g_{\w^{(k)}})\|, 
\\
\bepsilon^{(k+1)} = & \argmin_{\|\bepsilon \|_0 \leq t} \|\bepsilon - (\bepsilon^{(k)} - \eta\g_{\bepsilon^{(k)}})\|,
\end{align*}
where $\eta$ is the learning rate, and the projection steps for $\w$ and $\bepsilon$ have closed form solutions.

While the convergence of PSG for convex optimization has been well studied - the generic convergence rate is $O(1/\sqrt{K})$, its convergence for nonconvex optimization has rarely been studied until very recently. Thanks for the method developed in \citep{nguyen2014linear}, one can follow their method to establish the convergence rate of PSG for \eqref{auc-smo}. 
Omitting tedious statements and proofs, we simply state the results in below: under some mild conditions, PSG converges to a ball of the optimal solution to \eqref{auc-smo}:
\begin{equation*}
\label{theo2}
\mathbb{E}\|\v^{(k)}-\v^*\|_2\leq\kappa^{k}\|\v^{(0)}-\v^*\|_2+\frac{\sigma_s + \sigma^*}{1-\kappa},
\end{equation*}
where $\v = [\w, \bepsilon]^\top$ and $\v^*$ are the optimal solution to \eqref{auc-smo}. $\kappa$ is a number smaller than $1$. It depends on the restricted condition number of the objective of \eqref{auc-smo}. The radius of the ball depends on two terms: $\sigma_s$ and $\sigma^*$. $\sigma_s$ is the variance due to the use of the ``stochastic'' gradient, while $\sigma^*$ is the observation noise while collecting the data.

{\noindent \bf Computational time analysis}
\bc{We discuss the computational time of the proposed algorithm. The AUC optimization is essentially a ranking based algorithm. For such kind of method, the computation complexity increases through pairing the data. So the computational time is highly dependent on the data size. Fortunately, from the proposed Theorem \ref{thm:VC}, we can see that when the number of positive labeled data (i.e. $n_{+}$) is fixed the marginal gain by including more unlabeled data (i.e. increase $n$) is very minor. Since the dataset in PU learning is highly biased, i.e., $n_{+}$ is small, it indicates that we can still achieve a good result by only using a moderate amount of data. So we suggest a useful way for deciding the sample size of the unlabeled set in practice, e.g., when the number of unlabeled samples is substantially more than the number of positive samples, it is not necessary to include all the unlabeled samples in early training stages to avoid heavy computational burden.  In practice, one can gradually increase the size of unlabeled samples until $\widehat{\AUC}$ does not change significantly. 
}

\section{Theoretical Guarantee}
\label{theory}

This section will study the theoretical performances of the proposed model and algorithm.

\begin{theorem} \label{thm:VC}
Given two datasets $\X_+$ and $\X$, let $n_+=|\X_+|$ and $n=|\X|$, and assume that all data points in $\X_+$ are i.i.d samples from the distribution $\D_+$, and all data points in $\X$ are i.i.d samples from the distribution $\D$, with probability at least $1-\delta$ we have:
\begin{equation*}
\begin{aligned}
&\widehat{\AUC}(\w) - \AUC(\w)\le O\left(\frac{1}{{1 - \pi }}\left(\sqrt {\frac{\phi_{n}}{n}}
+ \sqrt {\frac{\phi_{n_+}}{{n_+}}} \right)\right),
\end{aligned}
\end{equation*}
where \begin{equation*}
\widehat{\AUC}=(1-\pi)^{-1}(\widehat{\BAUC}-\pi /2),
\end{equation*}
and $\phi_{n'}$ is defined as 
\begin{align*}
&{s\log (pn'/s) + \log \frac{1}{\delta }} \quad (\H = \H_s), 
\\
&{s\max_{G\in \mathcal{G}} |G|\log {n'\over s\max_{G\in \mathcal{G}}|G|} + s \log |\mathcal{G}| + \log \frac{1}{\delta }} \quad (\H = \H_{\mathcal{G}, s}),
\\ 
&\left(\sum_{G\in \mathcal{G}}\s_G\right) \log {n' \over \sum_{G\in \mathcal{G}}\s_G} + |\mathcal{G}|\max_{G\in \mathcal{G}} \s_G\log |G|+\log \frac{1}{\delta} \quad (\H = \H_{\mathcal{G}, \s}),
\end{align*}
where $p$ is the total number of features.
\end{theorem}
\begin{proof}
To prove the error bound between $\widehat{\AUC}$ and $\AUC$, we start to measure the difference between $\widehat{\BAUC}$ and $\BAUC$. First from the definition of $\BAUC$
\begin{equation*}
\widehat{\BAUC}(\w)=\frac{1}{{nn_+}}\sum\limits_{\scriptstyle \x_ +   \in \X_ +   \hfill \atop \scriptstyle \x \in \X \hfill} {\1(\w^\top \x_ +   \le \w^\top \x)},
\end{equation*}
we can take its expectation
\begin{equation*}
\begin{aligned}
\E[\widehat{\BAUC}(\w)]&=\E[\frac{1}{{nn_ +  }}\sum\limits_{\scriptstyle \x_ +   \in \X_ +  \atop \scriptstyle \x \in \X } {\1(\w^\top \x_ +   \le \w^\top \x)}] \\
&=\frac{1}{{nn_+}}\sum\limits_{\scriptstyle \x_ +   \in \X_ +  \atop \scriptstyle \x \in \X } { \mathbb{E}_{\scriptstyle \x_ +  \sim \D_ +   \hfill \atop \scriptstyle \x \sim \D \hfill} [\1(\w^\top \x_ +   \le \w^\top \x)]} \\
&={ \E_{\scriptstyle \x_ +  \sim \D_ +   \hfill \atop \scriptstyle \x \sim  \D \hfill} [\1(\w^\top \x_ +   \le \w^\top \x)]} \\
&=\BAUC(\w).
\end{aligned} 
\end{equation*}
Then we apply the connection between $\AUC$ and $\BAUC$ in Theorem~\ref{thm:AUCandBAUC} to obtain the following connection:
\begin{equation}
\label{dependence}
\widehat{\AUC}(\w)-\AUC(\w)=\frac{1}{1-\pi}(\widehat{\BAUC}(\w)-\BAUC(\w)).
\end{equation}
For simplicity, we use $\1(\cdot)$ to denote $\1(\w^\top \x_ +   \le \w^\top \x)$ in the following. Next we estimate the probabilistic error bound between $\BAUC$ and $\widehat{\BAUC}$.
\begin{align}
\nonumber
&\mathbb{P}(\widehat{\BAUC}(\w)-\BAUC(\w)\geq\epsilon)
\\ \nonumber
&=\mathbb{P}\Bigg (\frac{1}{{nn_ +  }}\sum\limits_{\scriptstyle \x_ +   \in \X_ +   \hfill \atop \scriptstyle \x \in \X \hfill} {{\bf 1}(\cdot)}-\frac{1}{{n_ +  }}\sum\limits_{\scriptstyle \x_ +   \in \X_ +} {\mathbb{E}_{\scriptstyle \x   \sim \D} {\bf 1}(\cdot)} 
\\ \nonumber
&\quad+\frac{1}{{n_ +  }}\sum\limits_{\scriptstyle \x_ +   \in \X_ +} {\mathbb{E}_{\scriptstyle \x  \sim \D} {\bf 1}(\cdot)}-{\mathbb{E}_{\scriptstyle \x_ +  \sim \D_ +   \hfill \atop \scriptstyle \x \sim  \D \hfill} {\bf 1}(\cdot)} \geq \epsilon \Bigg )
\\ \nonumber
&\leq \mathbb{P}\Bigg(\frac{1}{{nn_+}}\sum\limits_{\scriptstyle \x_ +   \in \X_ +   \hfill \atop \scriptstyle \x \in\X \hfill} {{\bf 1}(\cdot)}-\frac{1}{{n_ +  }}\sum\limits_{\scriptstyle \x_ +   \in \X_ +} { \mathbb{E}_{\scriptstyle \x   \sim \D} {\bf 1}(\cdot)} \geq \frac{\epsilon}{2}\Bigg) 
\\ \nonumber
&\quad+\mathbb{P}\Bigg(\frac{1}{{n_ +  }}\sum\limits_{\scriptstyle \x_ +   \in \X_ +} { \mathbb{E}_{\scriptstyle \x   \sim \D} {\bf 1}(\cdot)}-{ \mathbb{E}_{\scriptstyle \x_ +  \sim \D_ +   \hfill \atop \scriptstyle \x \sim  \D \hfill} {\bf 1}(\cdot)} \geq \frac{\epsilon}{2}\Bigg)
\\ \nonumber
& \quad \quad \quad
\text{(From union bound)}
\\
&\leq\mathbb{P}\Bigg(\max_{\X_+\atop \w\in \mathcal{H}} \left\{\frac{1}{{n}}\sum\limits_{\scriptstyle \x \in \X \hfill} {{\bf 1}(\cdot)}- \mathbb{E}_{\scriptstyle \x  \sim \D} {\bf 1}(\cdot)\right\} \geq \frac{\epsilon}{2}\Bigg) +\quad \label{eq:proof:A}
\\
&\quad\mathbb{P}\Bigg(\max_{\X\atop \w\in \mathcal{H}}\left\{ \frac{1}{{n_+}}\sum\limits_{\scriptstyle \x_+ \in \X_+ \hfill} {{\bf 1}(\cdot)}- \mathbb{E}_{\scriptstyle \x_+  \sim \D_+} {\bf 1}(\cdot)\right\} \geq \frac{\epsilon}{2}\Bigg). \label{eq:proof:B}
\end{align}
We first consider the case $\H = \H_s$. We only provide the upper bound for \eqref{eq:proof:A} (The upper bound for \eqref{eq:proof:B} can be obtained similarly.)
\begin{equation}
\eqref{eq:proof:A} \leq \gamma(\H_s, 2n, p) \times \mathbb{P}\Bigg(\frac{1}{{n}}\sum\limits_{\scriptstyle \x \in \X \hfill} {{\bf 1}(\cdot)}- \mathbb{E}_{\scriptstyle \x  \sim {\rm D}} {\bf 1}(\cdot) \geq \frac{\epsilon}{2}\Bigg).
\label{eq:proof:A1}
\end{equation}

Fixing $\w \in \mathcal{H}$ and $\x$ in \eqref{eq:proof:A}, we have
\begin{equation}
\label{eqfixw}
\begin{aligned}
&\mathbb{P}\Bigg(\frac{1}{{n}}\sum\limits_{\scriptstyle \x \in \X \hfill} {{\bf 1}(\cdot)}- \mathbb{E}_{\scriptstyle \x  \sim {\rm D}} {\bf 1}(\cdot) \geq \frac{\epsilon}{2}\Bigg) \leq \exp(-n\epsilon^2),
\end{aligned}
\end{equation}
which is bounded from Hoeffding's inequality. $\gamma(\H_s,2n,p)$ is defined as the \# of possible configuration of  ${\bf 1}(\cdot)$ on $2n$ points in $\mathbb{R}^p$ for $\H_s$. We have 
\begin{equation}
\label{eqgam}
\begin{aligned}
&\gamma(\mathcal{H}_s,2n,p) \leq \left( {\begin{matrix}
   p  \\ s \\ \end{matrix}} \right)\gamma(\mathbb{R}^s,2n,s) \leq  \left({\begin{matrix} p\\s\\\end{matrix}}\right)\left(\frac{2e2n}{s}\right)^s \leq  p^s \cdot (4e n/s)^s= (4epn/s)^s,
\end{aligned}
\end{equation}
where the second inequality uses the VC dimension for linear classifier \citep{vapnik2006estimation}. So we get the upper bound for \eqref{eq:proof:A} by \eqref{eqfixw} and \eqref{eqgam}
\begin{equation}
\begin{aligned}
\eqref{eq:proof:A} \leq (4epn/s)^s\cdot \exp (-n\epsilon^2).
\label{eq:proof:C}
\end{aligned}
\end{equation}
Similarly we can obtain the upper bound for \eqref{eq:proof:B}
\begin{equation}
\begin{aligned}
\eqref{eq:proof:B} \leq (4epn_+/s)^s\cdot \exp (-n_+\epsilon^2).
\label{eq:proof:D}
\end{aligned}
\end{equation}
So we can get \eqref{eq:proof:E} by \eqref{eq:proof:C} and \eqref{eq:proof:D}:
\begin{equation}
\begin{aligned}
&\mathbb{P}(\widehat{\BAUC}(\w)-\BAUC(\w)\geq\epsilon) \leq O((pn/s)^s\cdot \exp (-n\epsilon^2)+ (pn_+/s)^s\cdot \exp (-n_+\epsilon^2)).
\label{eq:proof:E}
\end{aligned}
\end{equation}
Let the right hand side in the above inequality be bounded by $\delta$, we have 
\begin{equation*}
\begin{aligned}
&\widehat{\BAUC}(\w)-\BAUC(\w)\leq O\Bigg(\sqrt{\frac{\phi_{n}}{n}}
+\sqrt{\frac{\phi_{n_+}}{n_+}}\Bigg),
\end{aligned}
\end{equation*}
holds with probability less than $\delta$, and $\phi_{n'}:={s\log (pn'/s) + \log \frac{1}{\delta }}$.
Using the dependence in \eqref{dependence}, we obtain
\begin{equation*}
\begin{aligned}
&\widehat{\AUC}(\w) - \AUC(\w)\le O\Bigg(\frac{1}{{1 - \pi }}\Bigg(\sqrt {\frac{\phi_{n}}{n}} + \sqrt {\frac{\phi_{n_+}}{{n_ +  }}}\Bigg )\Bigg),
\end{aligned}
\end{equation*}
with probability at least $1-\delta$.

To show the bound for $\H = \H_{\mathcal{G}, s}$, we only need to estimate \eqref{eq:proof:A1} by taking $\H$ as $\H_{\mathcal{G}, s}$ and estimate the upper bound for $\gamma(\H_{\mathcal{G},s}, 2n, p)$
\begin{align*}
\gamma(\H_{\mathcal{G},s}, 2n, p) \leq &
\left( {\begin{matrix}
   |\mathcal{G}| \\ s \\ \end{matrix}} \right)\gamma(\mathbb{R}^{s\max_{G\in \mathcal{G}}|G|},2n, s\max_{G\in \mathcal{G}}|G|)
   \leq 
   |\mathcal{G}|^s \left({2e2n \over s\max_{G\in \mathcal{G}}} \right)^{s \max_{G\in \mathcal{G}} |G|}. 
\end{align*}
Then we can follow the proof for $\H=\H_s$ to obtain the bound for $\H_{\mathcal{G}, s}$. 

We conclude the proof by considering the last case $\H = \H_{\mathcal{G}, \s}$. Following the same idea before, we only need to estimate \eqref{eq:proof:A1} by taking $\H$ as $\H_{\mathcal{G}, \s}$ and estimate the upper bound for $\gamma(\H_{\mathcal{G},\s}, 2n, p)$
\begin{align*}
\gamma(\H_{\mathcal{G},\s}, 2n, p) \leq &
\gamma(\mathbb{R}^{\sum_{G\in \mathcal{G}}\s_G},2n,\sum_{G\in \mathcal{G}}\s_G)\prod_{G\in \mathcal{G}}\left( {\begin{matrix}
   |G| \\ \s_G \\ \end{matrix}} \right)
   \\
   \leq & 
   \left({2e2n \over \sum_{G\in \mathcal{G}} \s_G}\right)^{\sum_{G\in \mathcal{G}} \s_G} \prod_{G\in \mathcal{G}} |G|^{\s_G}  
\\
\leq &
   \left({4en \over \sum_{G\in \mathcal{G}} \s_G}\right)^{\sum_{G\in \mathcal{G}} \s_G} \left(\max_{G\in \mathcal{G}} |G|^{\s_G}\right)^{|\mathcal{G}|}.
\end{align*}
It completes the proof.
\end{proof}
This theorem provides the upper bound of the difference between the empirically obtained $\widehat{\AUC}$ and the true $\AUC$. This leads to the following interesting observations:
\begin{itemize} 
\item When the number of unlabeled samples is more than the positive samples, the improvement on this bound is quite limited by increasing the number of unlabeled samples. 
\item The complexity of $\w$ affects this bound significantly. Let us consider the case $\H=\H_s$. Note that the error bound linearly depends on the sparsity parameter $s$. When $\min(n, n_+) = O(s \log p)$, the error bound converges to zero, which coincides with the consistency analysis for sparse signal recovery (for example, \citep{zhang2011sparse}). Actually, the error bounds for $\H=\H_{\mathcal{G}, s}$ and $\H = \H_{\mathcal{G}, \s}$ suggest similar observations. 
\item It is worthy of pointing out that if the super group set $\mathcal{G} = \{\{1\}, \{2\}, \cdots, \{p\}\}$ contains $p$ singleton groups, then it is known that $\H_{s} = \H_{\mathcal{G}, s}$ and the provided error bounds for $\H$ and $\H_{\mathcal{G}, s}$ are the same. And if the super set $\mathcal{G}= \{ \{1,2,\cdots, p\}\}$ only contains a single group and $\s_{G} = s$, then it is known that $\H_{s} = \H_{\mathcal{G}, \s}$ and the suggested error bounds are the same as well.

\item In addition, we can see that both $\H_{\mathcal{G}, \s}$ and $\H_{\mathcal{G}, s}$ have better error bound (i.e. smaller $\phi_{n'}$) than $\H_{s}$ when the number of nonzero elements are the same. Suppose we have $|\mathcal{G}|=k$ groups, and all the groups $G\in \mathcal{G}$ have the same size $p/k$ (for convenience, suppose $p$ is dividable by $k$). Let all the models have no more than $s$ nonzero elements. In this case, we compare the models under the sparsity hypotheses $\H_{s}$, $\H_{\mathcal{G}, sk/p}$ and $\H_{\mathcal{G}, \s}$ where $\s_G = s/k,\ \forall G\in \mathcal{G}$. According to Theorem~\ref{thm:VC}, we know that
$$\phi_{n'}(\H_{\mathcal{G}, sk/p})=s(\log(n'/s) + (k/p)\log k) + \log(1/\delta)$$ and 
$$\phi_{n'}(\H_{\mathcal{G}, \s}) = s(\log(n'/s) + \log (p/k)) + \log(1/\delta).$$
While $\phi_{n'}(\H_{s})=s(\log(n'/s) + \log p) + \log(1/\delta)$. Therefore, we have that $\phi_{n'}(\H_{\mathcal{G}, \s}) \leq \phi_{n'}(\H_{s})$ and $\phi_{n'}(\H_{\mathcal{G}, sk/p}) \leq \phi_{n'}(\H_{s})$.
\end{itemize}

\section{Experiments}
\label{experiments}

In this section, we will thoroughly evaluate the proposed $\BAUC$-OF model. First, we will test how the number of unlabeled training samples affects the $\AUC$ value to validate our theoretical analysis in Theorem~\ref{thm:VC} using synthetic data.  
Then, we compare the proposed model with the error minimization model using both synthetic data and real datasets. Finally, we further apply the proposed method on two real-world applications, the prediction of surgical site infection (SSI) and detection of seizure.

\begin{figure}[h]
\begin{center}
\includegraphics[width=0.8\textwidth]{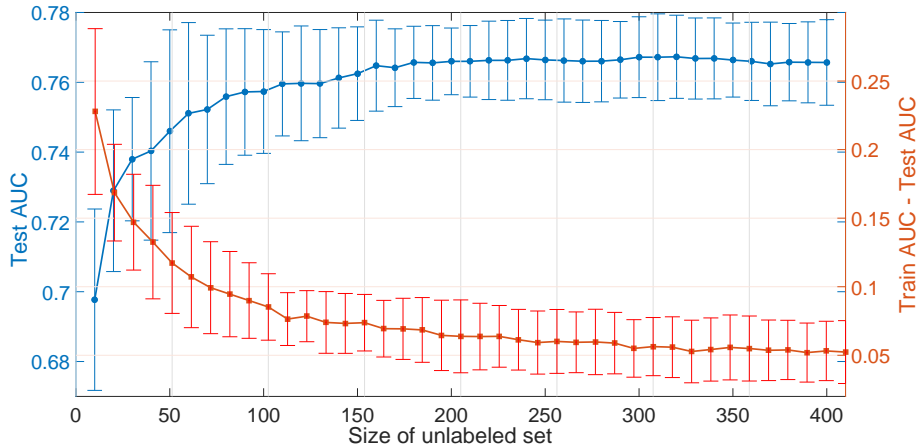}
\caption{\label{fig1}$\AUC$ curve (blue) and $\widehat{\AUC} - \AUC$ curve (red) with different sizes of the unlabeled set.}
\vspace{-6mm}
\end{center}
\end{figure}
\subsection{Empirical Validation of Theorem ~\ref{thm:VC}}

This section conducts empirical experiments to study how the number of unlabeled samples affects the $\AUC$ value, and evaluate the difference between the $\AUC$ and the empirical $\widehat{\AUC}$. Here, $\widehat{\AUC}$ is calculated using the training data, and $\AUC$ is calculated using the testing data with a sufficient number of samples. The positive samples follow the Gaussian distribution $\mathcal{N}(\1,7I)$ while the negative samples follow the Gaussian distribution $\mathcal{N}(-\1,7I)$. Note that this case is not linearly separable. The number of positive samples is fixed as $n_+ =50$. We also generate $n$ unlabeled samples where 10 \% of them are generated from the distribution model of positive samples. We gradually increase the size of unlabeled data from $n= 10$ to $n=200$. All experiments are repeated 10 times to obtain the mean and variance of the performance metrics.

We apply the PSG algorithm to solve the $\BAUC$ model (without feature selection and outlier detection) on synthetic datasets with various sizes of unlabeled sets. Results are reported in Figure~\ref{fig1}, i.e., the two curves correspond to the $\AUC$ and $\widehat{\AUC} - \AUC$, respectively. It indicates that when the number of unlabeled samples is more than 5 times of the number of positive samples, the improvement on $\AUC$ becomes quite minor and the estimation error between $\widehat{\AUC}$ and $\AUC$ does not change too much. This observation is consistent with our analysis in Theorem~\ref{thm:VC}. It essentially suggests that it is not necessary to include all the unlabeled data in the training process when the unlabeled data points are substantially more than the positive samples.


\subsection{Comparison of the Proposed Model~\eqref{auc-smo} with the State-of-the-Art Methods} \label{exp1}

This section compares the proposed model with the error minimization model and its variants. Specifically, the comparison involves \bc{9} algorithms: \bc{SVM under the ideal case (true labels of negative samples are known), one-class SVM (popular one-class classification algorithm) \cite{manevitz2001one},}  Biased SVM \citep{hoi2004biased} (the state-of-the art algorithm), $\ERR$ (the leading algorithm recently developed in \cite{plessis2015convex}, that is, \eqref{error-out} without outlier detection and feature selection), $\ERR$-O (\cite{plessis2015convex} + the proposed outlier detection), $\ERR$-OF (\cite{plessis2015convex} + the proposed outlier detection and feature selection),  $\BAUC$ (\eqref{auc-smo} without outlier detection and feature selection), $\BAUC$-O (\eqref{auc-smo} only with outlier detection), and $\BAUC$-OF in \eqref{auc-smo} (the complete version of the proposed model). \bc{Among them, the SVM under the ideal case serves as the baseline or gold standard where the true labels of the negative samples are known. It serves as the standard for us to evaluate the performance of all the other algorithms.}

{\bf Synthetic datasets:} 
We consider the binary classification task. 
{Each feature vector $\x\in \R^p$ contains $s$ relevant features while the remaining is irrelevant.
The relevant features are generated from $\mathcal{N}(1,7^2)$ and $\mathcal{N}(-1,7^2)$ for positive and negative samples respectively. Irrelevant features are generated from $\mathcal{N}(0,5^2)$. The wrong positive samples (outliers) are generated from $\mathcal{N}(-10,1)$. The training set contains 100 positive samples (containing wrong samples or outliers) and 300 unlabeled samples (20 positive samples + 280 negative samples). This leads to 30000 pairs in the $\AUC$ formulation. Test samples are generated in the same way with 1200 positive samples and 2800 negative samples. We varied the number of outliers and the number of features $p$ to compare all algorithms. 
}

\begin{table*}[htp!]

\begin{center}
\caption{AUC value (\%) comparison among nine algorithms on synthetic datasets with the different number of features (\#F). The number of true features is fixed to 40; the number of outliers is fixed to 7.}
\label{table:1}

\resizebox{\columnwidth}{!}{%
\begin{tabular}{c|c|cccccccc} 
 \hline
\#F& SVM(ideal) & One-class SVM & BSVM & ERR & BAUC & ERR-O &  ERR-OF & BAUC-O & BAUC-OF \\ 
 \hline\hline
  40 & \bc{88.01$\pm$0.60} &\bc{66.58$\pm$7.04} & 85.58$\pm$2.21 &84.14$\pm$1.51& 85.10$\pm$1.39 & 88.07$\pm$0.71& 87.97$\pm$0.72 & 88.02$\pm$0.69 & \textbf{88.21$\pm$0.63} \\ 
 \hline
  80 & 86.60$\pm$0.90 &63.56$\pm$6.58 & 78.17$\pm$2.31 &78.45$\pm$1.72& 80.60$\pm$1.56 & 86.00$\pm$0.84& 86.16$\pm$0.77 & 85.98$\pm$0.90 & \textbf{86.22$\pm$1.64} \\
 \hline
 120 &84.39$\pm$1.04 & 66.84$\pm$5.12 & 75.11$\pm$1.76 &76.03$\pm$1.71& 78.50$\pm$1.67 & 84.45$\pm$0.76& 84.65$\pm$1.02& 84.64$\pm$0.69& \textbf{86.28$\pm$1.07} \\
 \hline
 160 &83.01$\pm$1.01 & 62.46$\pm$4.55 & 70.52$\pm$1.86 &72.32$\pm$2.36& 75.02$\pm$2.42 & 82.40$\pm$1.28& 82.81$\pm$1.71& 82.39$\pm$1.37& \textbf{83.03$\pm$2.08} \\ 
 \hline
 200 &82.30$\pm$1.30 &60.49$\pm$4.01 & 70.12$\pm$1.52 &71.11$\pm$0.87& 74.11$\pm$0.78 & 81.43$\pm$0.52& 82.94$\pm$1.63& 81.29$\pm$0.66& \textbf{83.77$\pm$2.22} \\  
 \hline
 240 &81.37$\pm$1.59 &60.05$\pm$5.66 & 68.85$\pm$1.64 &68.64$\pm$2.87& 71.69$\pm$2.46 & 79.55$\pm$1.50& 81.47$\pm$2.36& 79.36$\pm$1.60& \textbf{81.54$\pm$2.36} \\ 
 \hline
 280 &80.44$\pm$1.44 &60.69$\pm$5.29 & 66.93$\pm$1.65 &68.55$\pm$2.04& 71.16$\pm$1.84 & 79.20$\pm$1.33& \textbf{81.50$\pm$1.51}& 78.86$\pm$1.12& 81.02$\pm$2.22 \\ 
 \hline
 320 &79.09$\pm$1.58 &59.12$\pm$3.02 & 66.85$\pm$1.25 &68.19$\pm$1.92& 70.81$\pm$2.00 & 78.60$\pm$1.26& 81.08$\pm$1.09& 78.49$\pm$1.13& \textbf{81.47$\pm$1.22} \\ 
 \hline
 360 &79.02$\pm$1.24 &57.24$\pm$5.03 & 63.77$\pm$0.83 &65.53$\pm$1.96& 68.19$\pm$2.02 & 76.47$\pm$1.54& \textbf{79.61$\pm$2.29}& 76.64$\pm$1.75& 78.44$\pm$2.20 \\ 
 \hline
400 &76.94$\pm$1.54 & 57.42$\pm$6.32 & 63.09$\pm$1.02 &65.39$\pm$1.94& 67.99$\pm$1.12 & 76.13$\pm$1.54& 78.27$\pm$2.34& 75.92$\pm$1.47& \textbf{78.77$\pm$3.10} \\ 
 \hline
\end{tabular}}
\end{center}
\end{table*}

\begin{table*} [htpb!]

\begin{center}
\caption{AUC value (\%) comparison among seven algorithms on synthetic datasets with the different number of outliers (\#O). The number of true features is fixed to 40; the number of redundant features is fixed to 160. }
\label{table:2}
\resizebox{\columnwidth}{!}{%
 \begin{tabular}{c|c|cccccccc} 
 \hline
\#O& SVM(ideal) & One-class SVM&BSVM& ERR& BAUC & ERR-O & ERR-OF  & BAUC-O & BAUC-OF \\ 
 \hline\hline
  1 & 84.10$\pm$0.85 & 80.30$\pm$2.40 &77.13$\pm$1.74& 79.84$\pm$1.89& 80.21$\pm$1.70 & 80.94$\pm$1.72& 83.14$\pm$2.53 & 80.88$\pm$1.64 & \textbf{83.93$\pm$2.30} \\ 
 \hline
 2 & 83.90$\pm$1.13& 78.06$\pm$2.80  &75.94$\pm$2.03& 78.97$\pm$0.74& 79.58$\pm$0.91 & 81.33$\pm$0.73& 83.19$\pm$1.21& 81.28$\pm$0.85 & \textbf{84.02$\pm$1.82} \\ 
 \hline
 3 & 82.79$\pm$0.87&77.60$\pm$2.13  &73.22$\pm$1.25& 77.75$\pm$1.43& 78.56$\pm$1.23 & 81.38$\pm$1.15& 83.80$\pm$1.35 & 81.29$\pm$1.06 & \textbf{84.17$\pm$1.20} \\
 \hline
 4 & 83.52$\pm$0.83 &75.37$\pm$3.36  &72.67$\pm$2.51& 76.29$\pm$1.61& 77.78$\pm$1.39 & 81.41$\pm$1.31& 83.07$\pm$2.05 & 81.24$\pm$1.37 & \textbf{83.49$\pm$1.67} \\
 \hline
 5 &83.42$\pm$1.22 &71.64$\pm$4.48  &71.59$\pm$1.69& 75.07$\pm$1.95& 77.13$\pm$1.76 & 81.57$\pm$1.32& 83.75$\pm$2.01 & 81.52$\pm$1.39 & \textbf{83.91$\pm$1.95} \\ 
 \hline
 6 &82.41$\pm$1.53 &67.30$\pm$6.58  &69.12$\pm$2.13& 73.70$\pm$2.35& 76.17$\pm$2.33 & 81.71$\pm$1.60& 83.74$\pm$2.26& 81.64$\pm$1.40 & \textbf{83.77$\pm$2.44} \\  
 \hline
 7 &82.10$\pm$1.63 & 61.47$\pm$5.1  &67.09$\pm$1.49& 70.41$\pm$2.87& 73.29$\pm$2.45 & 80.89$\pm$1.83& 82.86$\pm$2.41& 80.97$\pm$1.77& \textbf{83.30$\pm$2.90} \\ 
 \hline
 8 &81.63$\pm$0.87 & 57.70$\pm$4.59 &65.67$\pm$1.69& 69.03$\pm$2.16& 71.86$\pm$1.94 & 81.22$\pm$1.19& 83.51$\pm$2.01& 81.18$\pm$1.48& \textbf{83.77$\pm$1.73} \\ 
 \hline
\end{tabular}}
\end{center}
\end{table*}

\begin{table*} [htpb!]

\begin{center}
\caption{AUC value (\%) comparison among seven algorithms on real datasets. The first five rows correspond the dataset Arrhythmia with different setup.}
\label{table:3}
\resizebox{\columnwidth}{!}{%
\begin{tabular}{c|c|cccccccc} 
 \hline
  Datasets & SVM(ideal)&One-class SVM& BSVM& ERR &ERR-O &ERR-OF & BAUC & BAUC-O & BAUC-OF \\
 \hline\hline
  disease 1 vs disease 2& 97.08 & 57.42 &89.86 & 95.69 & 95.69 & 97.08 & 95.55 & 95.55 & \textbf{97.22} \\ 
 \hline
  disease 1 vs health& 91.74 & 61.27 &88.10& 89.27 & 90.00 & \textbf{94.99} & 88.52 & 88.52 & 94.94 \\ 
 \hline
  disease 2 vs health& 87.70 & 60.28 &87.00& 92.58 & 92.58 & \textbf{93.71} & 92.47 & 92.47 & 93.58 \\ 
 \hline
 health vs disease 1,2& 84.09 & 51.22 &78.14& 77.33 & 77.33 & 83.17 & 77.19 & 77.30 & \textbf{83.21} \\ 
 \hline
 health vs all& 78.05 & 53.25 &76.54& 77.27 & 77.28 & \textbf{77.91} & 77.39 & 77.39 &77.74 \\ 
 \hline
 SPECTF& 81.43 & 52.30 & 79.31 &80.00 & 80.00 & 80.66 & 80.24 & 80.36 & \textbf{80.73} \\ 
 \hline
 Readmission& 73.22 & 71.48 &72.89& 72.60 & 72.60 & 72.74 & 72.80 & 72.80 & \textbf{72.92} \\ 
 \hline
 Readmission(outlier)& 67.40 & 70.73 &64.00& 67.50 & 71.11 & 71.14 & 67.61 & 72.38 & \textbf{72.42}\\
 \hline
 Hill-Valley& 96.39 & 52.34 &84.15& 88.17 & 88.19 & 88.19 & 94.66 & 95.82 & \textbf{95.82}\\
 \hline
 Hill-Valley(noise)& 87.30 & 53.43 &80.73& 81.28 & 81.51 & 81.51 & 84.28 & 84.29 & \textbf{84.29} \\ \hline
\end{tabular}}
\end{center}
\vspace{-4mm}
\end{table*}

{\bf Real datasets:} Five real datasets are used to validate the proposed model, including Arrhythmia, SPECTF Heart, Readmission, noiseless Hill-Valley, and noise Hill-Valley \citep{Lichman:2013}.
The first real data set is the Arrhythmia data from the UCI data set. By choosing different groups of labels as positive class and negative class, we get five learning scenarios as shown in Table~\ref{table:3}. In this dataset, label 1, 2 and 10 are chosen as health, disease type 2, and disease type 1 respectively. \bc{ The reason to choose these three labels is that the number of people in these classes is large enough. The sizes of the training sets for five learning scenarios are $40, 100, 100, 60, 100$ respectively, with the number of positive data being $20$ in all these sets.} The second dataset is the SPECTF Heart Data Set. We choose label 0 as positive class. The size of training set is 80 with 50\% positive class. \bc{ The third and fourth data sets (i.e. Readmission and Readmission (outlier) in Table~\ref{table:3}) are generated from medical readmission dataset. In our experiments, we randomly choose 20 positive samples (no readmission) and 30 negative samples to form the training set. For the experiment with outliers (i.e. fourth data set), we randomly add 3 negative samples into the training set.} The last two datasets are the noisy version and noiseless version of the Hill-Valley dataset. We randomly choose 50 positive samples (Hill) and 150 negative samples (Valley) to form our training set, and the rest samples are used for testing. Note that, for all the training sets, we use the true class prior in the label error-based algorithm. At last, during the experiments, $75\%$ of the randomly chosen positive data inside the training sets is known to the algorithms.


{\bf Parameter tuning:} Without specification, we use the following way to tune the hyperparameters in our model \eqref{auc-smo}. In the experiment, we choose $\H=\H_s$. There are four hyperparameters: $\alpha$, $\beta$, $t$ (the outlier upper bound), and $s$ (the feature sparsity). Since $\alpha$ and $\beta$ are just used to restrict the magnitude of $\w$ and $\epsilon$, the performance is less sensitive to these two hyperparameters. So in practice, they are chosen to be small values, e.g., $\alpha = \beta = 0.001$.
$t$ and $s$ are important to the performance. \bc{They serve the same purpose as the weight of $L1$ norm sparse regularization, but they are discrete and much easier to tune. We initialize both $t$ and $s$ by small integers (e.g., $t=0$ and $s=5\%$ of the total number of features) and increase the value of each hyperparameter in a greedy manner, until that the performance on training set stops improving.}

The results for synthetic data are shown in Table~\ref{table:1} and Table~\ref{table:2}. Results for real data are shown in Table~\ref{table:3}.
Overall, the performance of the proposed $\BAUC$-OF model outperforms other models. \bc{The performance of all the PU learning algorithms without feature selection and outlier detection declines greatly when redundant features and outliers exist. Comparing to the traditional classification problems (ideal SVM), the performance of PU learning algorithms decreases more rapidly. Thus we can conclude PU learning is much more sensitive towards irrelevant features and noise inside the dataset. Intuitively, when all kinds of the uncertainties (unknown labels, irrelevant features, and outliers) combined and correlated together, the problem becomes much more complicated than the summation of those separated problems.}  While the proposed feature selection and outlier detection are included in the learning process, the performance is improved significantly. \bc{One-class classification is very sensitive towards the outliers as seen in the tables since it totally relies on the observed positive labeled data to make decisions.} 
\bc{For the real datasets, the performance of the SVM under the ideal case usually is the best, but it can be deteriorated by the outliers.} The benefit of feature selection is significant in Readmission (outlier) data because there are outliers (false positive samples) in the training sets. For the datasets, e.g., Arrhythmia, SPECTF Heart where the number of features is not large and may contain no outliers, \bc{all algorithms except for the one-class SVM tend to achieve the same performance. As before, the performance of one-class SVM can be very sensitive to problem types. In all datasets $\BAUC$-OF achieves very similar performance as the ideal case indicating that the proposed method acts as a powerful tool in dealing with real-world problems.}

\subsection{Real-world Application I: Prediction of Surgical Site Infection}

The PU problems are common in healthcare areas. For instance, here, we study the performance of our proposed method on a prediction problem for surgical site infection (SSI). It has been a very important question to predict the SSI onset based on some risk factors and wound characteristics, however, it is usually difficult to identify all the SSI patients since that may need us to keep track of the patients who have had surgery for quite a while. It is not uncommon that many patients' final status (whether or not they develop SSI) is unknown.

This results in a typical PU problem. In our study, we have 464 subjects in total, while for each subject, 37 clinical variables (such as some wound characteristics including the induration amount, wound edge distance, and wound edge color; physiological factors such as heart rate, diastolic RR and systolic RR) are measured in multiple time points. To test our algorithm, we split the data into training and testing data, while the training data consist of 80 subjects, i.e., 40 positive samples (infected) and 40 negative samples (not infected). Further, for the training data, only 35 positive subjects are assumed to be known to us, and the rest of the subjects form the unlabeled sample. For every subject, we use the measurement of the first 9 days after surgery, resulting in a total number of features for each subject as $37\times9 = 333$. Correspondingly, group sparsity is used in our algorithm for feature selection. By employing cross-validation in a wide range of choices on the tuning parameters, we finally choose 7 groups of features and 1 outlier in the experiments. we gather performances of the competing algorithm as shown in Figure~\ref{fig:SSI}, which suggests that the proposed $\BAUC$-OF outperforms other algorithms.

\begin{figure}[h]
\begin{center}
\includegraphics[width=0.9\textwidth]{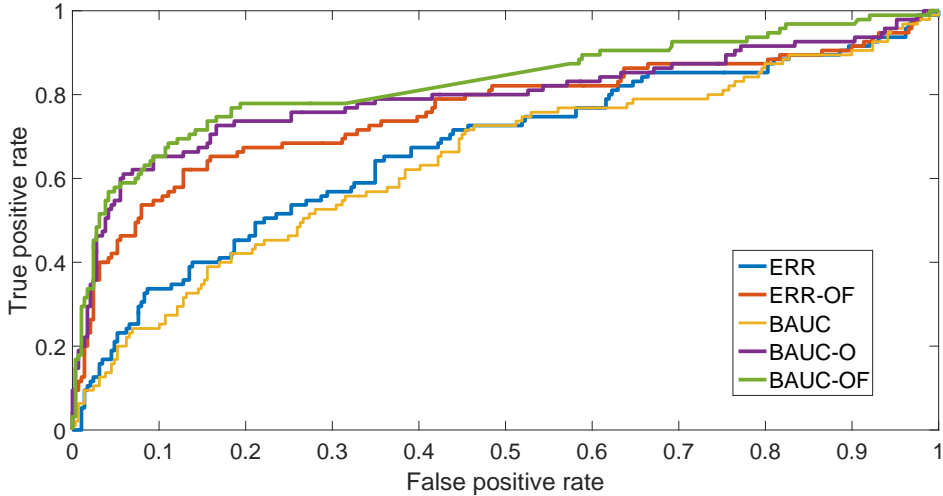}
\caption{ROC curves of the algorithms for SSI data.}
\vspace{-6mm}
\label{fig:SSI}
\end{center}
\end{figure}


\subsection{Real-world Application II: Seizure Detection from EEG Signals} \label{eegd}
Automatic seizure detection from EEG signals has been very important in seizure prevention and control. While the EEG signals provide rich information which can be leveraged to build prediction models, it is a time-consuming task to employ domain experts to segment the massive EEG signals and assign labels to the segments. It is not uncommon that manual labeling can only be applied to a few segments, resulting in a typical PU problem.
The EEG dataset was acquired from 8 epileptic mice (4 males and 4 females) at 10-14 weeks of age at Baylor College of Medicine. EEG recording electrodes (Teflon-coated silver wire, $127\mu m$ diameter) were placed in frontal cortex, somatosensory cortex, hippocampal CA1, and dentate gyrus. Spontaneous EEG activity (filtered between 0.1 Hz and 1 kHz, sampled at 2 kHz) were recorded in freely moving mice for 2 hours per day over 3 days. An example was shown in Figure~\ref{fig:samp}.
\begin{figure}[h]
\begin{center}
\includegraphics[width=0.9\textwidth]{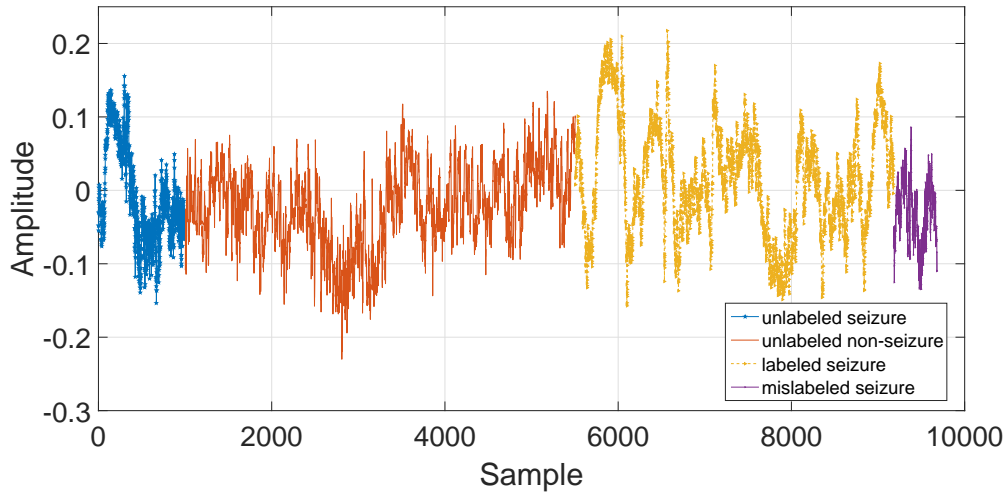}
\caption{An exemplary segment of the EEG signal. Yellow signals are ``seizure'' labeled correctly. Purple signals are not ``seizure'' but mislabeled. Blue signals are ``seizure'' that are not labeled. Red signals are non-seizure (also unlabeled).}
\vspace{-3mm}
\label{fig:samp}
\end{center}
\end{figure}

The EEG sequence for each mouse is with 261673 continuous time points (or signals), among which 21673 signals are labeled as seizure. At each time point, we extract 264 features for each signal including non-linear energy, FFT, RMS value, zero-crossing, Hjorth parameters, and entropy \citep{greene2008comparison} using the window length 2056.
\begin{figure}[h]
\begin{center}
\includegraphics[width=0.9\textwidth]{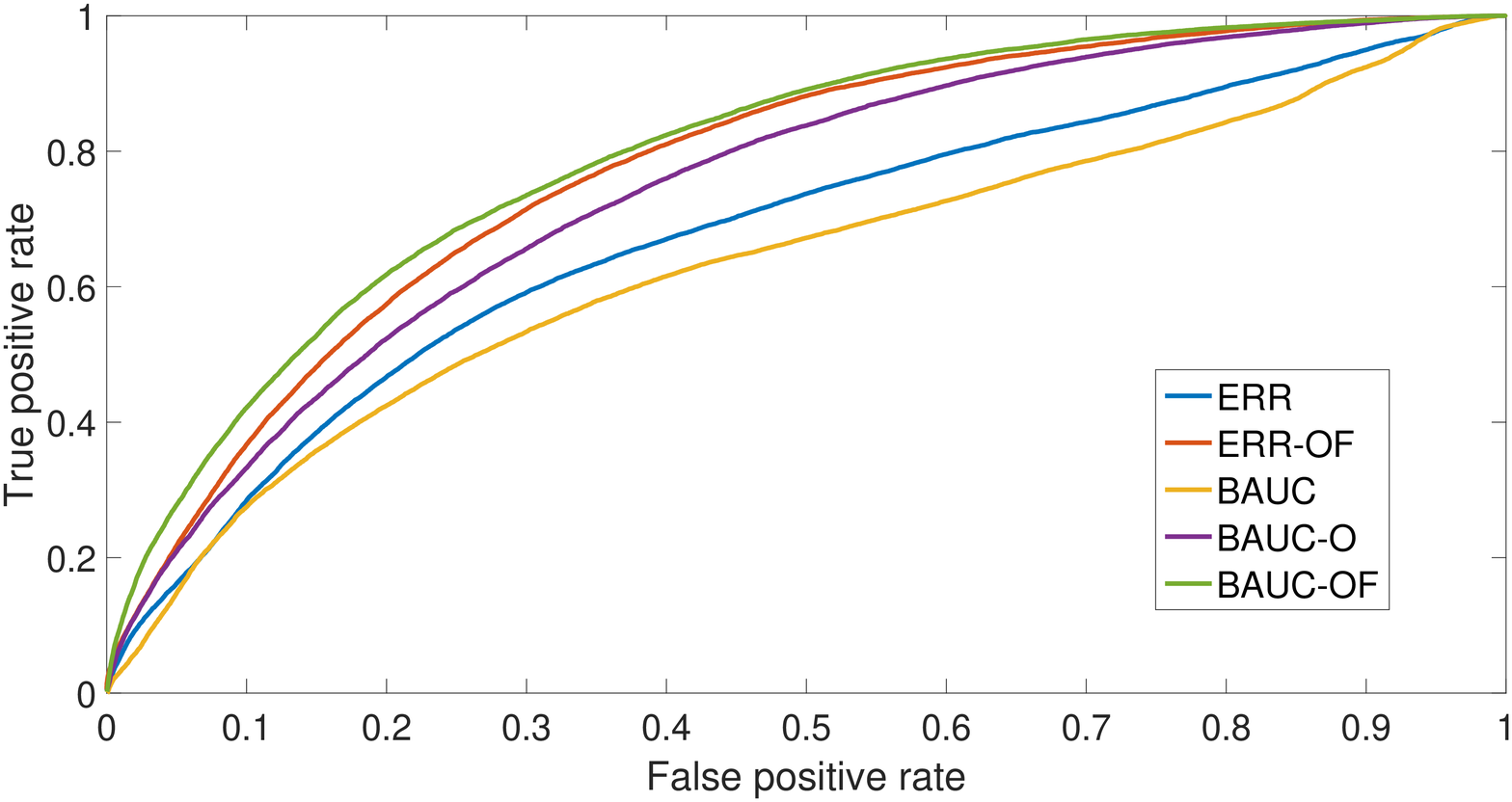}
\caption{ROC curves of the algorithms for EEG data.}
\label{fig:EEG}
\end{center}
\vspace{-6mm}
\end{figure}
 We construct the training dataset with 140 labeled positive signals and 160 unlabeled signals. The remaining forms the testing set. By employing cross-validation in a wide range of choices on the tuning parameters, we gather performances of the competing algorithm as shown in Figure~\ref{fig:EEG}.{ Finally, $5$ outliers and $80$ features are selected. And most features coming from FFT are excluded. This shows the feature selection actually works because only signals of some specific frequencies matter in this case.} It is apparent that the proposed $\BAUC$-OF outperforms other algorithms. 

\section{Conclusion}
\label{conclusion} 
Learning robust classifiers from positive and unlabeled data (PU problem) is a very challenging problem.  
In this paper, we propose a robust formulation to systematically address the challenging issues for PU problem. We unify AUC maximization, outlier detection, and feature selection in an integrated formulation, and study its theoretical performance that reveals insights about the relationships between the generalized error bounds with some important guidelines for practice. Extensive numerical studies using both synthetic data and real-world data demonstrate the superiority and efficacy of the proposed method compared with other state-of-the-art methods.

\newpage
\bibliographystyle{abbrvnat}
\bibliography{reference}

\begin{thebibliography}{61}
\providecommand{\natexlab}[1]{#1}
\providecommand{\url}[1]{\texttt{#1}}
\expandafter\ifx\csname urlstyle\endcsname\relax
  \providecommand{\doi}[1]{doi: #1}\else
  \providecommand{\doi}{doi: \begingroup \urlstyle{rm}\Url}\fi

\bibitem[Aggarwal and Yu(2001)]{aggarwal2001outlier}
C.~C. Aggarwal and P.~S. Yu.
\newblock Outlier detection for high dimensional data.
\newblock \emph{SIGMOD}, 2001.

\bibitem[Blanchard et~al.(2010)Blanchard, Lee, and Scott]{blanchard2010semi}
G.~Blanchard, G.~Lee, and C.~Scott.
\newblock Semi-supervised novelty detection.
\newblock \emph{The Journal of Machine Learning Research}, 2010.

\bibitem[Brefeld and Scheffer(2005)]{brefeld2005auc}
U.~Brefeld and T.~Scheffer.
\newblock Auc maximizing support vector learning.
\newblock In \emph{ICML workshop}, 2005.

\bibitem[Breunig et~al.(2000)Breunig, Kriegel, Ng, and Sander]{breunig2000lof}
M.~M. Breunig, H.-P. Kriegel, R.~T. Ng, and J.~Sander.
\newblock Lof: identifying density-based local outliers.
\newblock In \emph{ACM sigmod record}, 2000.

\bibitem[Calvo et~al.(2007)Calvo, Larrañaga, and Lozano]{geo2007a}
B.~Calvo, P.~Larrañaga, and J.~A. Lozano.
\newblock Learning bayesian classifiers from positive and unlabeled examples.
\newblock \emph{Pattern Recognition Letters}, 2007.

\bibitem[Campbell and Allen(2015)]{campbell2015within}
F.~Campbell and G.~I. Allen.
\newblock Within group variable selection through the exclusive lasso.
\newblock In \emph{arXiv:1505.07517}, 2015.

\bibitem[Chapelle and Keerthi(2008)]{chapelle2008multi}
O.~Chapelle and S.~S. Keerthi.
\newblock Multi-class feature selection with support vector machines.
\newblock In \emph{Proceedings of the American statistical association}, 2008.

\bibitem[Cortes and Mohri(2004)]{cortes2004auc}
C.~Cortes and M.~Mohri.
\newblock Auc optimization vs. error rate minimization.
\newblock \emph{NIPS}, 2004.

\bibitem[De~Bie et~al.(2007)De~Bie, Tranchevent, Van~Oeffelen, and
  Moreau]{de2007kernel}
T.~De~Bie, L.-C. Tranchevent, L.~M. Van~Oeffelen, and Y.~Moreau.
\newblock Kernel-based data fusion for gene prioritization.
\newblock \emph{Bioinformatics}, 2007.

\bibitem[Du~Plessis and Sugiyama(2014)]{du2014class}
M.~C. Du~Plessis and M.~Sugiyama.
\newblock Class prior estimation from positive and unlabeled data.
\newblock \emph{IEICE TRANSACTIONS on Information and Systems}, 2014.

\bibitem[du~Plessis et~al.(2014)du~Plessis, Niu, and Sugiyama]{du2014analysis}
M.~C. du~Plessis, G.~Niu, and M.~Sugiyama.
\newblock Analysis of learning from positive and unlabeled data.
\newblock In \emph{NIPS}, 2014.

\bibitem[Elkan(2001)]{elkan2001foundations}
C.~Elkan.
\newblock The foundations of cost-sensitive learning.
\newblock \emph{IJCAI}, 2001.

\bibitem[Elkan and Noto(2008)]{elkan2008learning}
C.~Elkan and K.~Noto.
\newblock Learning classifiers from only positive and unlabeled data.
\newblock In \emph{SIGKDD}, 2008.

\bibitem[Friedman et~al.(2008)Friedman, Hastie, and
  Tibshirani]{friedman2008sparse}
J.~Friedman, T.~Hastie, and R.~Tibshirani.
\newblock Sparse inverse covariance estimation with the graphical lasso.
\newblock In \emph{Biostatistics}, 2008.

\bibitem[Greene et~al.(2008)Greene, Faul, Marnane, Lightbody, Korotchikova, and
  Boylan]{greene2008comparison}
B.~Greene, S.~Faul, W.~Marnane, G.~Lightbody, I.~Korotchikova, and G.~Boylan.
\newblock A comparison of quantitative eeg features for neonatal seizure
  detection.
\newblock \emph{Clinical Neurophysiology}, 2008.

\bibitem[Hanley and McNeil(1982)]{rad1982the}
J.~A. Hanley and B.~J. McNeil.
\newblock The meaning and use of the area under a receiver operating
  characteristic (roc) curve.
\newblock \emph{Radiology}, 1982.

\bibitem[He et~al.(2006)He, Cai, and Niyogi]{he2006laplacian}
X.~He, D.~Cai, and P.~Niyogi.
\newblock Laplacian score for feature selection.
\newblock In \emph{NIPS}, 2006.

\bibitem[Hodge and Austin(2004)]{hodge2004survey}
V.~Hodge and J.~Austin.
\newblock A survey of outlier detection methodologies.
\newblock \emph{Artificial intelligence review}, 2004.

\bibitem[Hoi et~al.(2004)Hoi, Chan, Huang, Lyu, and King]{hoi2004biased}
C.-H. Hoi, C.-H. Chan, K.~Huang, M.~R. Lyu, and I.~King.
\newblock Biased support vector machine for relevance feedback in image
  retrieval.
\newblock In \emph{IJCNN}, 2004.

\bibitem[Huang and Zhang(2010)]{huang2010the}
J.~Huang and T.~Zhang.
\newblock The benefit of group sparsity.
\newblock In \emph{The Annals of Statistics}, 2010.

\bibitem[Hubert and Branden(2003)]{hubert2003robust}
M.~Hubert and K.~V. Branden.
\newblock Robust methods for partial least squares regression.
\newblock \emph{Journal of Chemometrics}, 2003.

\bibitem[Jiang et~al.(2001)Jiang, Tseng, and Su]{jiang2001two}
M.-F. Jiang, S.-S. Tseng, and C.-M. Su.
\newblock Two-phase clustering process for outliers detection.
\newblock \emph{Pattern recognition letters}, 2001.

\bibitem[Jordaan and Smits(2004)]{jordaan2004robust}
E.~M. Jordaan and G.~F. Smits.
\newblock Robust outlier detection using svm regression.
\newblock \emph{IJCNN}, 2004.

\bibitem[Kaur and Wasan(2006)]{kaur2006empirical}
H.~Kaur and S.~K. Wasan.
\newblock Empirical study on applications of data mining techniques in
  healthcare.
\newblock In \emph{Journal of Computer Science}, 2006.

\bibitem[Knorr and Ng(1999)]{knorr1999finding}
E.~M. Knorr and R.~T. Ng.
\newblock Finding intensional knowledge of distance-based outliers.
\newblock In \emph{VLDB}, 1999.

\bibitem[Lee and Liu(2003)]{icml2003leaning}
W.~S. Lee and B.~Liu.
\newblock Learning with positive and unlabeled examples using weighted logistic
  regression.
\newblock In \emph{ICML}, 2003.

\bibitem[Li et~al.(2011)Li, Guo, and Elkan]{geo2011a}
W.~Li, Q.~Guo, and C.~Elkan.
\newblock A positive and unlabeled learning algorithm for one-class
  classification of remote-sensing data.
\newblock \emph{IEEE Transactions on Geoscience and Remote Sensing}, 2011.

\bibitem[Li and Liu(2003)]{li2003learning}
X.~Li and B.~Liu.
\newblock Learning to classify texts using positive and unlabeled data.
\newblock In \emph{IJCAI}, 2003.

\bibitem[Li et~al.(2009)Li, Philip, and Liu]{sdm2009positive}
X.~Li, S.~Y. Philip, and B.~Liu.
\newblock Positive unlabeled learning for data stream classification.
\newblock In \emph{SDM}, 2009.

\bibitem[Lichman(2013)]{Lichman:2013}
M.~Lichman.
\newblock {UCI} machine learning repository, 2013.
\newblock URL \url{http://archive.ics.uci.edu/ml}.

\bibitem[Liu et~al.(2002)Liu, Lee, Yu, and Li]{icml2002partially}
B.~Liu, W.~S. Lee, P.~S. Yu, and X.~L. Li.
\newblock Partially supervised classification of text documents.
\newblock In \emph{ICML}, 2002.

\bibitem[Liu et~al.(2013)Liu, Fujimaki, and Ye]{liu2013forward}
J.~Liu, R.~Fujimaki, and J.~Ye.
\newblock Forward-backward greedy algorithms for general convex smooth
  functions over a cardinality constraint.
\newblock \emph{ICML}, 2013.

\bibitem[Manevitz and Yousef(2001)]{manevitz2001one}
L.~M. Manevitz and M.~Yousef.
\newblock One-class svms for document classification.
\newblock \emph{Journal of machine Learning research}, 2001.

\bibitem[Mason and Graham(2002)]{qua2002areas}
S.~J. Mason and N.~E. Graham.
\newblock Areas beneath the relative operating characteristics (roc) and
  relative operating levels (rol) curves: Statistical significance and
  interpretation.
\newblock \emph{Quarterly Journal of the Royal Meteorological Society}, 2002.

\bibitem[Mordelet and Vert(2014)]{mordelet2014bagging}
F.~Mordelet and J.-P. Vert.
\newblock A bagging svm to learn from positive and unlabeled examples.
\newblock \emph{Pattern Recognition Letters}, 2014.

\bibitem[Moya and Hush(1996)]{moya1996network}
M.~M. Moya and D.~R. Hush.
\newblock Network constraints and multi-objective optimization for one-class
  classification.
\newblock \emph{Neural Networks}, 1996.

\bibitem[Ng(2004)]{ng2004feature}
A.~Y. Ng.
\newblock Feature selection, l 1 vs. l 2 regularization, and rotational
  invariance.
\newblock In \emph{ICML}, 2004.

\bibitem[Nguyen et~al.(2011)Nguyen, Li, and Ng]{ijcai2011positive}
M.~N. Nguyen, X.~L. Li, and S.~K. Ng.
\newblock Positive unlabeled leaning for time series classification.
\newblock In \emph{IJCAI}, 2011.

\bibitem[Nguyen et~al.(2014)Nguyen, Needell, and Woolf]{nguyen2014linear}
N.~Nguyen, D.~Needell, and T.~Woolf.
\newblock Linear convergence of stochastic iterative greedy algorithms with
  sparse constraints.
\newblock \emph{arXiv:1407.0088}, 2014.

\bibitem[Obozinski et~al.(2006)Obozinski, Taskar, and
  Jordan]{obozinski2006multi}
G.~Obozinski, B.~Taskar, and M.~Jordan.
\newblock Multi-task feature selection.
\newblock \emph{Statistics Department, UC Berkeley, Tech. Rep}, 2006.

\bibitem[Pelckmans and Suykens(2009)]{pelckmans2009transductively}
K.~Pelckmans and J.~A. Suykens.
\newblock Transductively learning from positive examples only.
\newblock In \emph{ESANN}, 2009.

\bibitem[Peng et~al.(2009)Peng, Lin, Sun, and Tsai]{peng2009healthcare}
Y.~T. Peng, C.~Y. Lin, M.~T. Sun, and K.~C. Tsai.
\newblock Healthcare audio event classification using hidden markov models and
  hierarchical hidden markov models.
\newblock In \emph{ICME}, 2009.

\bibitem[Plessis et~al.(2015)Plessis, Niu, and Sugiyama]{plessis2015convex}
M.~D. Plessis, G.~Niu, and M.~Sugiyama.
\newblock Convex formulation for learning from positive and unlabeled data.
\newblock In \emph{ICML}, 2015.

\bibitem[Rakotomamonjy(2004)]{rakotomamonjy2004optimizing}
A.~Rakotomamonjy.
\newblock Optimizing area under roc curve with svms.
\newblock In \emph{ROCAI}, 2004.

\bibitem[Scott and Blanchard(2009)]{scott2009novelty}
C.~Scott and G.~Blanchard.
\newblock Novelty detection: Unlabeled data definitely help.
\newblock In \emph{AISTATS}, 2009.

\bibitem[Smola et~al.(2009)Smola, Song, and Teo]{smola2009relative}
A.~Smola, L.~Song, and C.~H. Teo.
\newblock Relative novelty detection.
\newblock In \emph{AISTATS}, 2009.

\bibitem[Tibshirani(1996)]{tibshirani1996regression}
R.~Tibshirani.
\newblock Regression shrinkage and selection via the lasso.
\newblock \emph{Journal of the Royal Statistical Society. Series B
  (Methodological)}, 1996.

\bibitem[Tropp(2004)]{tropp2004greed}
J.~A. Tropp.
\newblock Greed is good: Algorithmic results for sparse approximation.
\newblock \emph{Information Theory, IEEE Transactions on}, 2004.

\bibitem[Vapnik(2006)]{vapnik2006estimation}
V.~Vapnik.
\newblock \emph{Estimation of dependences based on empirical data}.
\newblock Springer Science \& Business Media, 2006.

\bibitem[Walczak(1995)]{walczak1995outlier}
B.~Walczak.
\newblock Outlier detection in multivariate calibration.
\newblock \emph{Chemometrics and intelligent laboratory systems}, 1995.

\bibitem[Xu et~al.(2017)Xu, Nie, and Han]{xu2017feature}
J.~Xu, F.~Nie, and J.~Han.
\newblock Feature selection via scaling factor integrated multi-class support
  vector machines.
\newblock In \emph{IJCAI}, 2017.

\bibitem[Yamanishi and Takeuchi(2001)]{yamanishi2001discovering}
K.~Yamanishi and J.-i. Takeuchi.
\newblock Discovering outlier filtering rules from unlabeled data: combining a
  supervised learner with an unsupervised learner.
\newblock In \emph{SIGKDD}, 2001.

\bibitem[Yamanishi et~al.(2004)Yamanishi, Takeuchi, Williams, and
  Milne]{yamanishi2004line}
K.~Yamanishi, J.-I. Takeuchi, G.~Williams, and P.~Milne.
\newblock On-line unsupervised outlier detection using finite mixtures with
  discounting learning algorithms.
\newblock \emph{Data Mining and Knowledge Discovery}, 2004.

\bibitem[Yang et~al.(2012)Yang, Li, Mei, Kwoh, and Ng]{bioinf2012positive}
P.~Yang, X.~L. Li, J.~P. Mei, C.~K. Kwoh, and S.~k. Ng.
\newblock Positive-unlabeled learning for disease gene identification.
\newblock \emph{Bioinformatics}, 2012.

\bibitem[Yao et~al.(2017)Yao, Liu, Jiang, Han, and Han]{yao2017lle}
C.~Yao, Y.-F. Liu, B.~Jiang, J.~Han, and J.~Han.
\newblock Lle score: a new filter-based unsupervised feature selection method
  based on nonlinear manifold embedding and its application to image
  recognition.
\newblock \emph{IEEE Transactions on Image Processing}, 2017.

\bibitem[Yu et~al.(2002)Yu, Han, and Chang]{kdd2002PEBL}
H.~Yu, J.~Han, and K.~C. Chang.
\newblock Pebl: positive example based learning for web page classification
  using svm.
\newblock In \emph{SIGKDD}, 2002.

\bibitem[Yuan et~al.(2014)Yuan, Li, and Zhang]{yuan2013gradient}
X.-T. Yuan, P.~Li, and T.~Zhang.
\newblock Gradient hard thresholding pursuit for sparsity-constrained
  optimization.
\newblock \emph{ICML}, 2014.

\bibitem[Zhang et~al.(2010)Zhang, Huang, Huang, Yu, Li, and
  Metaxas]{zhang2010automatic}
S.~Zhang, J.~Huang, Y.~Huang, Y.~Yu, H.~Li, and D.~N. Metaxas.
\newblock Automatic image annotation using group sparsity.
\newblock \emph{CVPR}, 2010.

\bibitem[Zhang(2009)]{zhang2009adaptive}
T.~Zhang.
\newblock Adaptive forward-backward greedy algorithm for sparse learning with
  linear models.
\newblock In \emph{NIPS}, 2009.

\bibitem[Zhang(2011)]{zhang2011sparse}
T.~Zhang.
\newblock Sparse recovery with orthogonal matching pursuit under rip.
\newblock \emph{Information Theory, IEEE Transactions on}, 2011.

\bibitem[Zou(2006)]{zou2006the}
H.~Zou.
\newblock The adaptive lasso and its oracle properties.
\newblock In \emph{Journal of the American statistical association}, 2006.

\end{thebibliography}

\end{document}